\documentclass{article}

\usepackage{microtype}
\usepackage{graphicx}
\usepackage{subfigure}
\usepackage{booktabs} 

\usepackage{hyperref}

\usepackage[accepted]{icml2018}

\usepackage[scaled=.90]{helvet}
\usepackage{courier}

\icmltitlerunning{A Reductions Approach to Fair Classification}

\usepackage{booktabs} 
\usepackage{enumitem}

\usepackage{amsmath,amssymb,amsthm}
\usepackage{eucal}
\usepackage{mathtools,xspace}
\usepackage{bbm}
\usepackage[dvipsnames]{xcolor}
\usepackage{tikz}
\usetikzlibrary{shapes.geometric}
\usepackage{comment}
\newtheorem{example}{Example}

\newtheorem{lemma}{Lemma}
\newtheorem{theorem}{Theorem}
\newtheorem{assume}{Assumption}

\usepackage[]{color-edits}
\addauthor{aa}{Fuchsia}
\addauthor{ab}{OliveGreen}
\addauthor{md}{Cerulean}
\addauthor{jcl}{brown}
\addauthor{hw}{red}

\newcommand{\card}[1]{\lvert#1\rvert}
\newcommand{\bigCard}[1]{\bigl\lvert#1\bigr\rvert}

\newcommand{\set}[1]{\{#1\}}

\newcommand{\bigSet}[1]{\bigl\{#1\bigr\}}
\newcommand{\BigSet}[1]{\Bigl\{#1\Bigr\}}
\newcommand{\braces}[1]{\{#1\}}

\newcommand{\bigBracks}[1]{\bigl[#1\bigr]}
\newcommand{\BigBracks}[1]{\Bigl[#1\Bigr]}

\newcommand{\Bracks}[1]{\left[#1\right]}
\newcommand{\parens}[1]{(#1)}
\newcommand{\Parens}[1]{\left(#1\right)}
\newcommand{\bigParens}[1]{\bigl(#1\bigr)}
\newcommand{\BigParens}[1]{\Bigl(#1\Bigr)}

\newcommand{\given}{\mathbin{\vert}}
\newcommand{\bigGiven}{\mathbin{\bigm\vert}}
\newcommand{\BigGiven}{\mathbin{\Bigm\vert}}

\newcommand{\norm}[1]{\lVert#1\rVert}

\newcommand{\abs}[1]{\lvert#1\rvert}
\newcommand{\Abs}[1]{\left\lvert#1\right\rvert}
\newcommand{\bigAbs}[1]{\bigl\lvert#1\bigr\rvert}
\newcommand{\BigAbs}[1]{\Bigl\lvert#1\Bigr\rvert}

\newcommand{\Ehat}{\widehat{\mathbb{E}}}
\newcommand{\Phat}{\widehat{\mathbb{P}}}
\renewcommand{\P}{\mathbb{P}}
\newcommand{\err}{\textup{\mdseries err}}
\newcommand{\herr}{\widehat{\err}}
\newcommand{\E}{\ensuremath{\mathbb{E}}}
\newcommand{\calH}{\ensuremath{\mathcal{H}}}
\newcommand{\calF}{\ensuremath{\mathcal{F}}}
\newcommand{\calX}{\ensuremath{\mathcal{X}}}
\newcommand{\calZ}{\ensuremath{\mathcal{Z}}}

\DeclareMathOperator*{\argmin}{arg\,min}
\DeclareMathOperator*{\argmax}{arg\,max}
\newcommand{\ind}{\mathbf{1}}
\newcommand{\eps}{\varepsilon}
\newcommand{\cons}{\ensuremath{\gamma}}
\newcommand{\consh}{\ensuremath{\widehat{\cons}}}
\newcommand{\R}{\ensuremath{\mathbb{R}}}
\newcommand{\tO}{\widetilde{O}}

\newcommand{\vlambda}{\boldsymbol{\lambda}}
\newcommand{\lambdah}{\ensuremath{\widehat{\lambda}}}
\newcommand{\vlambdah}{\boldsymbol{\lambdah}}
\newcommand{\tvlambda}{\boldsymbol{\lambda}^\dagger}
\newcommand{\tQ}{Q^\dagger}

\definecolor{purple}{rgb}{0.6, 0.4, 0.8}
\definecolor{red}{rgb}{1.0, 0.03, 0.0}
\definecolor{orange}{rgb}{0.93, 0.57, 0.13}
\definecolor{yellow}{rgb}{0.94, 0.88, 0.19}
\definecolor{green}{rgb}{0.0, 0.55, 0.55}
\definecolor{blue}{rgb}{0.0, 0.75, 1.0}

\newcommand{\dep}{\textup{DP}\xspace}
\newcommand{\eo}{\textup{EO}\xspace}

\newcommand{\moment}{\ensuremath{\mu}}
\newcommand{\momentVec}{\boldsymbol{\moment}}
\newcommand{\momenth}{\ensuremath{\widehat{\moment}}}
\newcommand{\momentVech}{\boldsymbol{\momenth}}
\newcommand{\event}{\ensuremath{\mathcal{E}}}
\newcommand{\consM}{\ensuremath{\mathbf{M}}}
\newcommand{\consC}{\ensuremath{c}}
\newcommand{\consCvec}{\ensuremath{\mathbf{\consC}}}
\newcommand{\consCh}{\ensuremath{\widehat{\consC}}}
\newcommand{\consCvech}{\ensuremath{\widehat{\consCvec}}}

\newcommand{\Attr}{\ensuremath{\mathcal{A}}}
\newcommand{\vzero}{\mathbf{0}}
\newcommand{\hmu}{\widehat{\mu}}
\newcommand{\hvmu}{\widehat{\boldsymbol{\mu}}}
\newcommand{\vmu}{\boldsymbol{\mu}}
\newcommand{\hvgamma}{\widehat{\boldsymbol{\gamma}}}
\newcommand{\vgamma}{\boldsymbol{\gamma}}
\newcommand{\hvlambda}{\widehat{\boldsymbol{\lambda}}}
\newcommand{\hvc}{\widehat{\mathbf{c}}}
\newcommand{\ve}{\mathbf{e}}
\newcommand{\hc}{\widehat{c}}
\newcommand{\hQ}{\widehat{Q}}
\newcommand{\vtheta}{\boldsymbol{\theta}}
\newcommand{\hgamma}{\widehat{\gamma}}

\newcommand{\vr}{\mathbf{r}}
\newcommand{\trans}{^{\!\top}}

\newcommand{\eqn}[1]{equation~\eqref{eqn:#1}}
\newcommand{\eqns}[2]{equations~\eqref{eqn:#1} and~\eqref{eqn:#2}}
\newcommand{\Eqn}[1]{Equation~\eqref{eqn:#1}}
\newcommand{\Eqns}[2]{Equations~\eqref{eqn:#1} and~\eqref{eqn:#2}}
\newcommand{\Alg}[1]{Algorithm~\ref{alg:#1}}
\newcommand{\App}[1]{Appendix~\ref{app:#1}}
\newcommand{\Assume}[1]{Assumption~\ref{assume:#1}}
\newcommand{\Thm}[1]{Theorem~\ref{thm:#1}}
\newcommand{\Lemma}[1]{Lemma~\ref{lemma:#1}}
\newcommand{\Fig}[1]{Figure~\ref{fig:#1}}

\newcommand{\BestLambda}{\textsc{\mdseries Best}_{\vlambda}}
\newcommand{\BestH}{\textsc{\mdseries Best}_h}

\newcommand{\midx}{\ensuremath{j}}
\newcommand{\cidx}{\ensuremath{k}}
\newcommand{\mIdx}{\ensuremath{\mathcal{J}}}
\newcommand{\cIdx}{\ensuremath{\mathcal{K}}}

\newcommand{\Qh}{\ensuremath{\widehat{Q}}}

\newtheorem{definition}{Definition}

\newcommand\hide[1]{}

\begin{document}

\twocolumn[
\icmltitle{A Reductions Approach to Fair Classification}



\icmlsetsymbol{equal}{*}

\begin{icmlauthorlist}
\icmlauthor{Alekh Agarwal}{msr}
\icmlauthor{Alina Beygelzimer}{yah}
\icmlauthor{Miroslav Dud\'ik}{msr}
\icmlauthor{John Langford}{msr}
\icmlauthor{Hanna Wallach}{msr}
\end{icmlauthorlist}

\icmlaffiliation{msr}{Microsoft Research, New York}
\icmlaffiliation{yah}{Yahoo!\ Research, New York}

\icmlcorrespondingauthor{A.~Agarwal}{alekha@microsoft.com}
\icmlcorrespondingauthor{A.~Beygelzimer}{beygel@gmail.com}
\icmlcorrespondingauthor{M.~Dud\'ik}{mdudik@microsoft.com}
\icmlcorrespondingauthor{J.~Langford}{jcl@microsoft.com}
\icmlcorrespondingauthor{H.~Wallach}{wallach@microsoft.com}


\vskip 0.3in
]



\printAffiliationsAndNotice{}  

\begin{abstract}
  We present a systematic approach for achieving fairness
  in a binary classification setting. While we focus on two well-known
  quantitative definitions of fairness, our approach
  encompasses many other previously studied definitions as
  special cases. The key idea is to reduce fair classification to a
  sequence of cost-sensitive classification problems, whose
  solutions yield a randomized classifier with the lowest (empirical)
  error subject to the desired constraints. We introduce two reductions that
  work for any representation of the cost-sensitive
  classifier and compare favorably to prior baselines on
  a variety of data sets, while overcoming several of their disadvantages.\looseness=-1
\end{abstract}

\section{Introduction}
\label{sec:introduction}

Over the past few years, the media have paid considerable attention to
machine learning systems and their ability to inadvertently
discriminate against minorities, historically disadvantaged
populations, and other protected groups when allocating resources (e.g., loans) or opportunities (e.g., jobs). In response to this
scrutiny---and driven by ongoing debates and collaborations with
lawyers, policy-makers, social scientists, and
others~\citep[e.g.,][]{barocas16big}---machine learning researchers
have begun to turn their attention to the topic of
``fairness in machine learning,'' and, in particular, to the design
of fair classification and regression algorithms.

In this paper we study the task of binary classification subject to
fairness constraints with respect to a pre-defined protected
attribute, such as race or sex. Previous work in this area can be
divided into two broad groups of approaches.

The first group of approaches
incorporate specific quantitative
definitions of fairness into existing machine learning methods, often by relaxing the
desired definitions of fairness, and only enforcing weaker constraints,
such as lack of correlation~\citep[e.g.,][]{woodworth17,zafar2017fairness,johnson2016impartial,
  kamishima2011, DoniniEtAl18}. The resulting fairness guarantees typically
only hold under strong distributional assumptions, and the approaches
are tied to specific families of classifiers, such as SVMs.

The second group of approaches eliminate the restriction to specific classifier families
and treat the underlying classification method as a ``black box,''
while implementing a wrapper that either works by pre-processing the data
or post-processing the classifier's
predictions~\citep[e.g.,][]{kamiran12,feldman15certifying,hardt16,calmon17optimized}. Existing
pre-processing approaches are specific to particular definitions of
fairness and typically seek to come up with a single transformed data set that
will work across all learning algorithms, which,
in practice, leads to classifiers that still exhibit substantial unfairness
(see our evaluation in Section~\ref{sec:experiments}). In contrast, post-processing
allows a wider range of fairness definitions and results in provable
fairness guarantees. However, it is not guaranteed to find the most accurate fair classifier,
and requires test-time access to the protected attribute, which might
not be available.


We present a general-purpose approach
that has the key advantage of this second group of approaches---i.e.,
the underlying classification method is treated as a black box---but without the
noted disadvantages. Our approach encompasses a wide range of fairness
definitions, is guaranteed to yield the most accurate
fair classifier, and does not require test-time access to the
protected attribute. Specifically, our approach allows any definition of fairness that can be formalized
via linear inequalities on conditional moments, such as
\emph{demographic parity} or \emph{equalized odds}~(see
Section~\ref{sec:fairness}). We show how binary classification subject to these constraints can be reduced
to a sequence of cost-sensitive classification problems.
We require only black-box access to a cost-sensitive classification algorithm,
which does not need to have any
knowledge of the desired definition of fairness or protected attribute.
We show that the solutions to our sequence of
cost-sensitive classification problems yield a randomized classifier
with the lowest (empirical) error subject to the desired fairness
constraints.\looseness=-1

\citet{CorbettDaviesEtAl17} and \citet{menon18cost}
begin with a similar goal to ours, but they analyze
the Bayes optimal classifier under fairness constraints
in the limit of infinite data. In contrast, our focus is algorithmic,
our approach applies to any classifier family, and we obtain
finite-sample guarantees. \citet{Dwork17partition} also begin with a
similar goal to ours. Their approach partitions
the training examples into subsets according to protected attribute values
and then leverages transfer learning
to jointly learn from these separate data sets. Our approach avoids partitioning the data and assumes access only to a classification algorithm rather than a transfer learning algorithm.

A preliminary version of this paper appeared at the FAT/ML
workshop~\citep{AgarwalEtAl17}, and led to extensions with more
general optimization objectives~\citep{AlabiImKa18} and combinatorial
protected attributes~\citep{KearnsEtAl18}.


In the next section, we formalize our problem. While we focus
on two well-known quantitative definitions of
fairness, our approach also encompasses many
other previously studied definitions of fairness as special cases. In
Section~\ref{sec:reductions}, we describe our reductions approach to
fair classification and its guarantees in detail.
The experimental study in Section~\ref{sec:experiments} shows that our
reductions compare favorably to three baselines, while overcoming
some of their disadvantages and also offering the flexibility
of picking a suitable accuracy--fairness tradeoff.
Our results demonstrate the utility of having a general-purpose
approach for combining machine learning methods and quantitative fairness definitions.\looseness=-1

\section{Problem Formulation}
\label{sec:problem_formulation}

We consider a binary classification setting where the training
examples consist of triples $(X, A, Y)$, where $X\in\calX$ is a
feature vector, $A \in \Attr$ is a protected attribute, and $Y \in
\{0,1\}$ is a label. The feature vector $X$ can either contain the
protected attribute $A$ as one of the features or contain other
features that are arbitrarily indicative of $A$. For example, if the
classification task is to predict whether or not someone will default
on a loan, each training example might correspond to a person, where
$X$ represents their demographics, income level, past payment
history, and loan amount; $A$ represents their race; and $Y$
represents whether or not they defaulted on that loan. Note that $X$
might contain their race as one of the features or, for example,
contain their zipcode---a feature that is often correlated with
race. Our goal is to learn an accurate classifier $h:\calX\to \{0,1\}$
from some set (i.e., family) of classifiers $\calH$, such as
linear threshold rules, decision trees, or neural nets, while satisfying
some definition of fairness. Note that the classifiers in $\calH$ do
not explicitly depend on $A$.\looseness=-1

\subsection{Fairness Definitions}
\label{sec:fairness}

We focus on two well-known quantitative definitions of fairness that
have been considered in previous work on fair classification; however,
our approach also encompasses many other previously studied
definitions of fairness as special cases, as we explain at the end of
this section.

The first definition---\emph{demographic} (or statistical)
\emph{parity}---can be thought of as a stronger version of the US
Equal Employment Opportunity Commission's ``four-fifths rule,'' which
requires that the ``selection rate for any race, sex, or ethnic group
[must be at least] four-fifths (4/5) (or eighty percent) of the rate
for the group with the highest rate.''\footnote{See the Uniform
  Guidelines on Employment Selection Procedures, 29 C.F.R. \S
  1607.4(D) (2015).}

\begin{definition}[Demographic parity---\dep]
\label{defn:dp}
  A classifier $h$ satisfies demographic parity under a distribution
  over $(X, A, Y)$ if its prediction $h(X)$ is statistically
  independent of the protected attribute $A$---that is, if
  $\P[{h(X)=\hat{y}} \given {A=a}] = \P[h(X)=\hat{y}]$ for all $a$,
  $\hat{y}$.  Because $\hat{y}\in\set{0,1}$, this is equivalent to
  ${\E[h(X)\given A=a]}=\E[h(X)]$ for all $a$.
  \end{definition}

The second definition---\emph{equalized odds}---was recently proposed
by \citet{hardt16} to remedy two previously noted flaws with
demographic parity~\citep{dwork12}. First, demographic parity permits
a classifier which accurately classifies data points with one
value ${A=a}$,
such as the value $a$ with the most data, but makes random predictions
for data points with ${A\ne a}$ as long as the probabilities of
${h(X)=1}$ match. Second, demographic parity rules out perfect
classifiers whenever $Y$ is correlated with $A$. In contrast,
equalized odds suffers from neither of these flaws.\looseness=-1

\begin{definition}[Equalized odds---\eo]
\label{defn:eo}
  A classifier $h$ satisfies equalized odds under a distribution over
  $(X, A, Y)$ if its prediction $h(X)$ is conditionally independent of
  the protected attribute $A$ given the label $Y$---that is, if
  $\P[{h(X)=\hat{y}} \given {A=a}, {Y=y}] = \P[{h(X)=\hat{y}} \given
    {Y=y}]$ for all $a$, $y$, and $\hat{y}$.  Because
  $\hat{y}\in\set{0,1}$, this is equivalent to $\E[h(X)\given
    A=a,Y=y]=\E[h(X)\given Y=y]$ for all $a$, $y$.
\end{definition}


We now show how each definition can be viewed as a special case of a
general set of linear constraints of the form
\begin{equation}
  \label{eqn:general_form}
    \consM \momentVec(h) \leq \consCvec,
      \end{equation}
where matrix $\consM \in \R^{|\mathcal{K}| \times |\mathcal{J}|}$ and
vector $\consCvec\in\R^{|\mathcal{K}|}$ describe the linear
constraints, each indexed by $\cidx\in\cIdx$, and $\momentVec(h) \in
\R^{|\mathcal{J}|}$ is a vector of conditional moments of the form
\[
\mu_j(h) = \E\bigBracks{\,g_j(X, A, Y, h(X))\, \bigGiven \,\event_j\,}
\quad
  \text{for $j \in \mathcal{J}$},
\]
where
$g_j : {\calX\times\Attr\times\set{0,1}\times\set{0,1}} \to [0,1]$
and $\event_j$ is an event defined with
respect to $(X, A, Y)$. Crucially, $g_j$ depends on $h$, while
$\event_j$ cannot depend on $h$ in any way.

\begin{example}[\dep]  In a binary classification setting, demographic
  parity can be expressed as a set of $|\mathcal{A}|$ equality
  constraints, each of the form $\E[h(X) \given A = a] =
  \E[h(X)]$. Letting $\mathcal{J} = \mathcal{A} \cup \{ \star \}$,
  $g_j(X, A, Y, h(X)) = h(X)$ for all $j$, $\event_{a} = \{A = a\}$,
  and $\event_{\star} = \{\textit{True}\}$, where $\{\textit{True}\}$
  refers to the event encompassing all points in the sample space,
  each equality constraint can be expressed as $\mu_a(h) =
  \mu_{\star}(h)$.\footnote{Note that $\mu_\star(h) = \E[h(X) \,|\, \textit{True}] = \E[h(X)]$.} Finally, because each such constraint can be
  equivalently expressed as a pair of inequality constraints of the
  form
\begin{align*}
  \mu_a(h) - \mu_{\star}(h) &\leq 0\\
  -\mu_a(h) + \mu_{\star}(h) &\leq 0,
\end{align*}
demographic parity can be expressed as \eqn{general_form}, where
$\mathcal{K} = \mathcal{A} \times \{+,-\}$, $M_{(a,+),a'} =
\ind\set{a' = a}$, $M_{(a,+),\star}=-1$, $M_{(a,-),a'} = -\ind\set{a'
  = a}$, $M_{(a,-),\star} = 1$, and $\consCvec =
\boldsymbol{0}$. Expressing each equality constraint as a pair of
inequality constraints allows us to control the extent to which each
constraint is enforced by positing $c_k>0$ for some (or all)
$k$.\looseness=-1
\end{example}

\begin{example}[\eo] In a binary classification setting, equalized
  odds can be expressed as a set of $2\,\card{\mathcal{A}}$ equality
  constraints, each of the form $\E[h(X) \given {A=a}, {Y\!=y}] =
  \E[h(X)\given {Y\!=y}]$. Letting $\mathcal{J} =
  (\mathcal{A}\cup\set{\star})\times\set{0,1}$, $g_j(X, A, Y, h(X)) =
  h(X)$ for all $j$, $\event_{(a,y)} = \set{A=a, {Y\!=y}}$, and
  $\event_{(\star, y)} = \set{Y\!=y}$, each equality constraint can be equivalently
  expressed as 
\looseness=-1
\begin{align*}
  \mu_{(a,y)}(h) - \mu_{(\star,y)}(h) &\leq 0\\
  -\mu_{(a,y)}(h) + \mu_{(\star,y)}(h) &\leq 0.
\end{align*}
As a result, equalized odds can be expressed as \eqn{general_form},
where $\mathcal{K} = {\mathcal{A} \times \mathcal{Y} \times \{+,-\}}$,
$M_{(a,y,+),(a',y')} = \ind\set{a'\!\!=\!a,\,y'\!\!=\!y}$,
$M_{(a,y,+),(\star,y')}=-1$, $M_{(a,y,-),(a',y')} =
-\ind\set{a'\!\!=\!a,\,y'\!\!=\!y}$, $M_{(a,y,-),(\star,y')} = 1$, and
$\consCvec= \vzero$. Again, we can posit $c_k>0$ for some (or all) $k$ to allow small violations of some (or all) of the constraints.
\end{example}

Although we omit the details, we note that many other previously
studied definitions of fairness can also be expressed as
\eqn{general_form}. For example, \emph{equality of
  opportunity}~\cite{hardt16} (also known as \emph{balance for the
  positive class}; \citealp{kleinberg2017inherent}), \emph{balance for
  the negative class}~\cite{kleinberg2017inherent}, \emph{error-rate
  balance}~\cite{chouldechova}, \emph{overall accuracy
  equality}~\cite{berk}, and \emph{treatment equality}~\cite{berk} can
all be expressed as \eqn{general_form}; in contrast,
\emph{calibration}~\cite{kleinberg2017inherent} and \emph{predictive
  parity}~\cite{chouldechova} cannot because to do so would require
the event $\event_j$ to depend on $h$. We note that our approach can
also be used to satisfy multiple definitions of fairness, though if
these definitions are mutually contradictory, e.g., as described
by~\citet{kleinberg2017inherent}, then our guarantees become vacuous.\looseness=-1

\subsection{Fair Classification}

In a standard (binary) classification setting, the goal is to learn
the classifier $h\in\calH$ with the minimum classification error:
$\err(h)\coloneqq \P[h(X) \neq Y]$.
However, because our goal is to
learn the most accurate classifier while satisfying fairness
constraints, as formalized above, we instead seek to find the solution
to the constrained optimization problem\footnote{%
We consider misclassification error for concreteness, but all the
results in this paper apply to any error of the
form $\err(h) = \E[g_\err(X, A, Y, h(X))]$, where
$g_\err(\cdot,\cdot,\cdot,\cdot)\in[0,1]$.}
%
\begin{equation}
  \min_{h \in \mathcal{H}} \textrm{err}(h)
  \quad
  \text{subject to}
  \quad
  \consM \momentVec(h) \leq \consCvec.
\end{equation}
Furthermore, rather than just considering classifiers in the set
$\mathcal{H}$, we can enlarge the space of possible classifiers by
considering \emph{randomized classifiers} that can be obtained via a
distribution over $\mathcal{H}$. By considering randomized
classifiers, we can achieve better accuracy--fairness tradeoffs than
would otherwise be possible. A randomized classifier $Q$ makes a
prediction by first sampling a classifier $h \in \mathcal{H}$ from $Q$
and then using $h$ to make the prediction. The resulting
classification error is $\err(Q)=\sum_{h \in \calH}
Q(h)\,\err(h)$ and the conditional moments are $\vmu(Q) =
\sum_{h \in \calH} Q(h)\vmu(h)$ (see \App{randomized} for the
derivation). Thus we seek to solve
\begin{equation}
\label{eqn:objective}
  \min_{Q \in \Delta} \err(Q)
\quad
  \text{subject to}
\quad
  \consM \momentVec(Q) \leq \consCvec,
\end{equation}
where $\Delta$ is the set of all distributions over $\mathcal{H}$.

In practice, we do not know the true distribution over $(X, A, Y)$ and
only have access to a data set of training examples $\{(X_i, A_i,
Y_i)\}_{i=1}^n$. We therefore replace $\textrm{err}(Q)$ and $\vmu(Q)$
in \eqn{objective} with their empirical versions $\herr(Q)$ and
$\hvmu(Q)$. Because of the sampling error in $\hvmu(Q)$, we also allow
errors in satisfying the constraints by setting $\widehat{c}_k = c_k +
\eps_k$ for all $k$, where $\eps_k \geq 0$. After these modifications,
we need to solve the empirical version of \eqn{objective}:
\begin{equation}
\label{eqn:objective_empirical}
  \min_{Q \in \Delta}\herr(Q)
\quad
  \text{subject to}
\quad
  \consM \hvmu(Q) \leq \hvc
.
\end{equation}

\section{Reductions Approach}
\label{sec:reductions}

We now show how the problem~\eqref{eqn:objective_empirical} can be
reduced to a sequence of \emph{cost-sensitive classification}
problems. We further show that the solutions to our sequence of
cost-sensitive classification problems yield a randomized classifier
with the lowest (empirical) error subject to the desired constraints.

\subsection{Cost-sensitive Classification}

We assume access to a cost-sensitive classification algorithm for the
set $\calH$. The input to such an algorithm is a data set of training
examples $\{(X_i, C_i^0, C_i^1)\}_{i=1}^n$, where $C_i^0$ and $C_i^1$
denote the losses---\emph{costs} in this setting---for predicting the
labels $0$ or $1$, respectively, for $X_i$. The algorithm outputs
\begin{equation}
  \argmin_{h \in \calH} \sum_{i=1}^n h(X_i)\,C_i^1 +
  (1-h(X_i))\,C_i^0.
\label{eqn:cs-defn}
\end{equation}
This abstraction allows us to specify different costs for different
training examples, which is essential for incorporating fairness
constraints. Moreover, efficient cost-sensitive classification
algorithms are readily available for several common classifier
representations~\citep[e.g.,][]{WAP, SECOC, adacost}. In particular,
\eqn{cs-defn} is equivalent to a \emph{weighted classification}
problem, where the input consists of labeled examples $\set{(X_i, Y_i,
  W_i)}_{i=1}^n$ with $Y_i\in\set{0,1}$ and $W_i\ge 0$, and the goal
is to minimize the weighted classification error $\sum_{i=1}^n
W_i\,\ind\braces{h(X_i)\ne Y_i}$. This is equivalent to \eqn{cs-defn}
if we set $W_i=\abs{C_i^0-C_i^1}$ and $Y_i=\ind\braces{C_i^0\ge
  C_i^1}$.

\subsection{Reduction}
\label{sec:reduction}

To derive our fair classification algorithm, we rewrite
\eqn{objective_empirical} as a saddle point problem. We begin by
introducing a Lagrange multiplier $\lambda_\cidx\ge 0$ for each of the
$|\cIdx|$ constraints, summarized as $\vlambda\in\R_+^{|\cIdx|}$, and
form the Lagrangian
\begin{align}
\notag
  L(Q,\vlambda)
&=
  \herr(Q)
  +\vlambda\trans\bigParens{\consM\hvmu(Q)-\hvc}.
\end{align}
Thus, \eqn{objective_empirical} is equivalent to
\begin{equation}
\label{eqn:objective_empirical_L}
  \adjustlimits\min_{Q\in\Delta}
  \max_{\;\;\;\vlambda\in\R_+^{|\cIdx|}\;\;\;}
  L(Q,\vlambda).
\end{equation}
For computational and statistical reasons, we impose an additional
constraint on the $\ell_1$ norm of $\vlambda$ and seek to
simultaneously find the solution to the constrained version of
\eqref{eqn:objective_empirical_L} as well as its dual, obtained by
switching min and max:
\begin{align}
\label{primal}
\tag{\textup{P}}
  \adjustlimits
  \min_{Q\in\Delta\strut}
  \max_{\;\;\;\vlambda\in\R_+^{|\cIdx|},\,\norm{\vlambda}_1\le B\;\;\;}
  L(Q,\vlambda),\\
\label{dual}
\tag{\textup{D}}
  \adjustlimits\max_{\vlambda\in\R_+^{|\cIdx|},\,\norm{\vlambda}_1\le B}
  \min_{\;\;\;Q\in\Delta\;\;\;\strut}
  L(Q,\vlambda).
\end{align}
Because $L$ is linear in $Q$ and $\vlambda$ and the domains of
$Q$ and $\vlambda$ are convex and compact, both problems have
solutions (which we denote by $Q^\dagger$ and $\tvlambda$) and the
minimum value of \eqref{primal} and the maximum value of \eqref{dual}
are equal and coincide with $L(\tQ,\tvlambda)$. Thus,
$(\tQ,\tvlambda)$ is the saddle point of~$L$ \citep[Corollary 37.6.2
  and Lemma 36.2 of][]{Rockafellar70}.\looseness=-1

We find the saddle point by using the standard scheme of
\citet{FreundSc96}, developed for the equivalent problem of solving
for an equilibrium in a zero-sum game.  From game-theoretic
perspective, the saddle point can be viewed as an equilibrium of a
game between two players: the $Q$-player choosing $Q$ and the
$\vlambda$-player choosing $\vlambda$. The Lagrangian $L(Q,\vlambda)$
specifies how much the $Q$-player has to pay to the $\vlambda$-player
after they make their choices. At the saddle point, neither player
wants to deviate from their choice.\looseness=-1

Our algorithm finds an approximate equilibrium in which neither player
can gain more than $\nu$ by changing their choice (where $\nu>0$ is an
input to the algorithm).  Such an approximate equilibrium corresponds
to a \emph{$\nu$-approximate saddle point} of the Lagrangian, which is
a pair $(\hQ,\hvlambda)$, where
\begin{align*}
  L(\hQ,\hvlambda)
  &\le L(Q,\hvlambda)+\nu
&&
  \text{for all $Q\in\Delta$},\\
  L(\hQ,\hvlambda)
  &\ge L(\hQ,\vlambda)-\nu
&&
  \text{for all $\vlambda\in\R_+^{|\cIdx|}$, $\norm{\vlambda}_1\le B$}.
\end{align*}
We proceed iteratively by running a no-regret algorithm for the
$\vlambda$-player, while executing the best response of the
$Q$-player. Following \citet{FreundSc96}, the average play of both
players converges to the saddle point.  We run the exponentiated
gradient algorithm~\citep{KivinenWa97} for the $\vlambda$-player and
terminate as soon as the suboptimality of the average play falls below
the pre-specified accuracy $\nu$. The best response of the $Q$-player
can always be chosen to put all of the mass on one of the candidate
classifiers $h\in\calH$, and can be implemented by a single call to a
cost-sensitive classification algorithm for the set $\calH$.

Algorithm~\ref{alg:EG:general} fully implements this scheme, except
for the functions $\BestLambda$ and $\BestH$, which correspond to the
best-response algorithms of the two players. (We need the best
response of the $\vlambda$-player to evaluate whether the
suboptimality of the current average play has fallen below $\nu$.) The
two best response functions can be calculated as follows.

\begin{algorithm}[t]
\caption{Exp.\ gradient reduction for fair classification}
\label{alg:EG:general}
  \begin{algorithmic}
    \STATE{Input:~~training examples $\braces{(X_i,Y_i,A_i)}_{i=1}^n$\\
    ~\hphantom{Input:}~fairness constraints specified by $g_\midx$, $\event_\midx$, $\consM$, $\hvc$\\
    ~\hphantom{Input:}~bound $B$, accuracy $\nu$, learning rate $\eta$}
    \STATE{Set $\vtheta_1=\vzero\in\R^{\card{\cIdx}}$}
    \FOR{$t=1, 2, \ldots$}
    \STATE{Set $\lambda_{t,\cidx} = B\,
                \frac{\exp\braces{\theta_\cidx}}{1+\sum_{\cidx'\in\cIdx} \exp\braces{\theta_{\cidx'}}}$
           for all $\cidx\in\cIdx$}
    \STATE{$h_t\gets\BestH(\vlambda_t)$}
    \STATE{$\hQ_t\gets\frac{1}{t}\sum_{t'=1}^t h_{t'},
      \quad
            \overline{L}\gets L\Parens{\hQ_t,\BestLambda(\hQ_t)}$}
    \STATE{$\hvlambda_t\gets\frac{1}{t}\sum_{t'=1}^t \vlambda_{t'},
           \quad
            \underline{L}\gets L\Parens{\BestH(\hvlambda_t),\hvlambda_t}$}
    \STATE{$\nu_t\gets\max\BigSet{L(\hQ_t,\hvlambda_t)-\underline{L},\quad
                                    \overline{L}-L(\hQ_t,\hvlambda_t)
                                  }$}
    \IF{$\nu_t\le\nu$}
          \STATE{Return $(\hQ_t,\hvlambda_t)$}
    \ENDIF
    \STATE{Set $\vtheta_{t+1}=\vtheta_t+\eta\left({\consM\hvmu(h_t)-\hvc}\right)$}
    \ENDFOR
\end{algorithmic}
\end{algorithm}

\paragraph{$\BestLambda(Q)$: the best response of the $\vlambda$-player.} The best
response of the $\vlambda$-player for a given $Q$ is any maximizer of
$L(Q,\vlambda)$ over all valid $\vlambda$s. In our setting, it can
always be chosen to be either $\vzero$ or put all of the mass on the most
violated constraint. Letting $\hvgamma(Q)\coloneqq\consM\hvmu(Q)$ and
letting $\ve_k$ denote the $k^\textrm{th}$ vector of the standard
basis, $\BestLambda(Q)$ returns\looseness=-1
\[
   \begin{cases}
   \vzero
      &\text{if $\hvgamma(Q)\le\hvc$},
   \\
   B\ve_{\cidx^*}
      &\text{otherwise, where $\cidx^*=\argmax_\cidx[\hgamma_\cidx(Q)-\hc_\cidx]$}.
   \end{cases}
\]

\paragraph{$\BestH(\vlambda)$: the best response of the $Q$-player.} Here, the best response minimizes $L(Q,\vlambda)$
over all $Q$s in the simplex. Because $L$ is linear in $Q$, the
minimizer can always be chosen to put all of the mass on a single
classifier $h$. We show how to obtain the classifier constituting the
best response via a reduction to cost-sensitive
classification. Letting $p_j\coloneqq\Phat[\event_j]$ be the empirical
event probabilities, the Lagrangian for $Q$ which puts all of the mass
on a single $h$ is then\looseness=-1
\begin{align*}
&
  L(h,\vlambda)
=
  \herr(h)
  +\vlambda\trans\bigParens{\consM\hvmu(h)-\hvc}
\\[3pt]
&\;{}
=
  \Ehat\bigBracks{\ind\set{h(X)\ne Y}}
  -\vlambda\trans\hvc
  +\!\sum_{k,j}
   M_{k,j}\lambda_k\hmu_j(h)
\\
&\;{}=
  -\vlambda\trans\hvc
  +\Ehat\bigBracks{\ind\set{h(X)\ne Y}}
\\[6pt]
&\;\quad{}
  +\!\sum_{k,j}
   \frac{M_{k,j}\lambda_k}{p_j}\Ehat\BigBracks{g_j\bigParens{X\!,\!A,\!Y\!,\!h(X)}\,\ind\set{(X\!,\!A,\!Y)\in\event_j}}
.
\end{align*}
%
%
Assuming a data set of training examples
$\set{(X_i,A_i,Y_i)}_{i=1}^n$, the minimization of $L(h,\vlambda)$
over $h$ then corresponds to cost-sensitive classification on
$\set{(X_i,C_i^0,C_i^1)}_{i=1}^n$ with costs%
\footnote{%
For general error, $\err(h) = \E[g_\err(X,A,Y,h(X))]$, the costs
$C_i^0$ and $C_i^1$ contain, respectively, the terms
$g_\err(X_i,A_i,Y_i,0)$ and $g_\err(X_i,A_i,Y_i,1)$
instead of $\ind\set{Y_i\ne 0}$ and $\ind\set{Y_i\ne 1}$.}
\begin{align*}
C_i^0
&
 =
  \ind\set{Y_i\ne 0}
\\
&\;\quad{}
  +\!\sum_{k,j}
   \frac{M_{k,j}\lambda_k}{p_j}g_j\parens{X_i,\!A_i,\!Y_i,0}\,\ind\set{(X_i,\!A_i,\!Y_i)\in\event_j}
\\
C_i^1
&
 =
  \ind\set{Y_i\ne 1}
\\
&\;\quad{}
  +\!\sum_{k,j}
   \frac{M_{k,j}\lambda_k}{p_j}g_j\parens{X_i,\!A_i,\!Y_i,1}\,\ind\set{(X_i,\!A_i,\!Y_i)\in\event_j}
.
\end{align*}
%

\begin{theorem}
\label{thm:alg}
Letting $\rho\coloneqq\max_h\norm{\consM\hvmu(h)-\hvc}_\infty$,
\Alg{EG:general} satisfies the inequality
\[
  \nu_t\le\frac{B\log(\card{\cIdx}+1)}{\eta t}+\eta\rho^2 B.
\]
Thus, for $\eta =\!\frac{\nu}{2\rho^2 B}$, \Alg{EG:general} will return a
$\nu$-approximate saddle point of $L$ in at most $\frac{4\rho^2
  B^2\log(\card{\cIdx}+1)}{\nu^2}$ iterations.\looseness=-1
\end{theorem}
This theorem, proved in \App{alg}, bounds the suboptimality $\nu_t$ of the average play
$(\hQ_t,\hvlambda_t)$, which is equal to its suboptimality as a
saddle point. The right-hand side of the bound is optimized by
$\eta=\sqrt{\log(\card{\cIdx}+1)}\,/\,(\rho\sqrt{t})$, leading to the
bound $\nu_t\le 2\rho B\sqrt{\log(\card{\cIdx}+1)\,/\,t}$. This bound
decreases with the number of iterations $t$ and grows very slowly
with the number of constraints $\card{\cIdx}$. The quantity $\rho$ is
a problem-specific constant that bounds how much any single classifier
$h \in \mathcal{H}$ can violate the desired set of fairness
constraints. Finally, $B$ is the bound on the $\ell_1$-norm of
$\vlambda$, which we introduced to enable this specific algorithmic
scheme. In general, larger values of $B$ will bring the problem
\eqref{primal} closer to \eqref{eqn:objective_empirical_L}, and thus
also to \eqref{eqn:objective_empirical}, but at the cost of needing
more iterations to reach any given suboptimality. In particular, as we
derive in the theorem, achieving suboptimality $\nu$ may need
up to $4\rho^2 B^2\log(\card{\cIdx}+1)\,/\,\nu^2$ iterations.\looseness=-1

\begin{example}[\dep]
Using the matrix $\consM$ for demographic parity as described in
Section~\ref{sec:problem_formulation}, the cost-sensitive reduction
for a vector of Lagrange multipliers $\vlambda$ uses costs
\begin{equation*}
  C_i^0 = \ind\braces{Y_i \ne 0},
\quad
  C_i^1 = \ind\braces{Y_i \ne 1}
          + \frac{\lambda_{A_i}}{p_{A_i}} - \!\sum_{a\in\Attr}\lambda_a,
\end{equation*}
where $p_a\coloneqq\smash{\Phat}[A=a]$ and
$\lambda_a\coloneqq\lambda_{(a,+)}-\lambda_{(a,-)}$, effectively
replacing two non-negative Lagrange multipliers by a single
multiplier, which can be either positive or negative. Because
$c_\cidx=0$ for all $k$, $\hc_\cidx=\eps_\cidx$. Furthermore, because
all empirical moments are bounded in $[0,1]$, we can assume
$\eps_\cidx\le 1$, which yields the bound $\rho\le 2$. Thus,
\Alg{EG:general} terminates in at most
$16B^2\log(2\,\card{\Attr}+1)\,/\,\nu^2$ iterations.
\end{example}

\begin{example}[\eo]
For equalized odds, the cost-sensitive
reduction for a vector of Lagrange multipliers $\vlambda$ uses costs
\begin{align*}
  C_i^0 &= \ind\braces{Y_i \ne 0},\\
  C_i^1 &= \ind\braces{Y_i \ne 1}
          + \frac{\lambda_{(A_i,Y_i)}}{p_{(A_i,Y_i)}}-\!\sum_{a\in\Attr}\frac{\lambda_{(a,{Y_i})}}{p_{(\star,Y_i)}},
\end{align*}
where $p_{(a,y)}\coloneqq\smash{\Phat}[{A=a},{Y\!=y}]$,
$p_{(\star,y)}\coloneqq\smash{\Phat}[{Y\!=y}]$, and
$\lambda_{(a,y)}\coloneqq\lambda_{(a,y,+)}-\lambda_{(a,y,-)}$.  If we
again assume ${\eps_\cidx\le 1}$, then we obtain the bound $\rho\le
2$. Thus, \Alg{EG:general} terminates in at most
$16B^2\log(4\,\card{\Attr}+1)\,/\,\nu^2$ iterations.
\end{example}

\subsection{Error Analysis}

Our ultimate goal, as formalized in \eqn{objective}, is to minimize
the classification error while satisfying fairness constraints under a
true but unknown distribution over $(X, A, Y)$. In the process of
deriving \Alg{EG:general}, we introduced three different sources of
error.  First, we replaced the true classification error and true
moments with their empirical versions. Second, we introduced a bound $B$ on the magnitude
of $\vlambda$. Finally, we only run the optimization algorithm for a
fixed number of iterations, until it reaches suboptimality level
$\nu$. The first source of error, due to the use of empirical rather than true quantities,
is unavoidable and constitutes the
underlying statistical error. The other two sources of error, the bound~$B$ and the suboptimality level~$\nu$,
stem from the optimization algorithm and can be driven arbitrarily small
at the cost of additional iterations. In this section, we show how
the statistical error and the optimization error
affect the true accuracy and the fairness of the randomized classifier
returned by \Alg{EG:general}---in other words, how well
\Alg{EG:general} solves our original problem~\eqref{eqn:objective}.\looseness=-1


To bound the statistical error, we use the Rademacher complexity of
the classifier family $\calH$, which we denote by $R_n(\calH)$, where
$n$ is the number of training examples. We assume that $R_n(\calH)\le
C n^{-\alpha}$ for some $C\ge 0$ and $\alpha\le 1/2$. We note that $\alpha = 1/2$
in the vast majority of classifier families, including
norm-bounded linear functions (see Theorem~1 of~\citealp{kakade2009complexity}), neural
networks (see Theorem~18 of~\citealp{BartlettMe02}),
and classifier families with bounded VC dimension (see Lemma~4 and Theorem~6
of~\citealp{BartlettMe02}).\looseness=-1


Recall that in our empirical optimization problem we assume
that $\consCh_k = \consC_k+\eps_k$, where $\eps_k\ge 0$ are error
bounds that account for the discrepancy between $\vmu(Q)$ and
$\hvmu(Q)$. In our analysis, we assume that these error bounds
have been set in accordance with the Rademacher complexity of $\calH$.
\begin{assume}
\label{assume:Rn}
There exists $C, C'\ge 0$ and $\alpha\le 1/2$ such that $R_n(\calH)\le
C n^{-\alpha}$ and $\eps_k=C'\sum_{\midx \in \mIdx} \abs{M_{k,j}}
n_\midx^{-\alpha}$, where $n_j$ is the number of data points that
fall in $\event_j$,
\[
  n_j\coloneqq\bigCard{\bigSet{i:\:(X_i, A_i, Y_i)\in\event_j}}.
\]
\end{assume}
The optimization error can be bounded via a careful analysis of the
Lagrangian and the optimality conditions of~\eqref{primal}
and~\eqref{dual}. Combining the three different sources of error
yields the following bound, which we prove in \App{stat}.
\begin{theorem}
\label{thm:stat}
  Let \Assume{Rn} hold for $C'\ge 2C+2+\sqrt{\ln(4/\delta)\,/\,2}$, where
  $\delta > 0$. Let $(\Qh, \vlambdah)$ be any $\nu$-approximate saddle
  point of $L$, let $Q^\star$ minimize $\err(Q)$ subject to
  $\consM\momentVec(Q) \leq \consCvec$, and let
  $p^\star_\midx=\P[\event_\midx]$. Then, with probability at least
  $1-(|\mIdx|+1)\delta$, the distribution $\smash{\Qh}$ satisfies
\begin{align*}
    \err(\Qh)
    &
    \le \err(Q^\star)
        + 2\nu + \tO(n^{\!-\alpha}),\\
\cons_\cidx(\Qh)
    &
     \le c_k + \frac{1\!+\!2\nu}{B} + \!\sum_{\midx \in \mIdx}
     \abs{M_{k,j}}\,\tO(n_\midx^{\!-\alpha}) && \text{for all $k$},
\end{align*}
where $\smash{\tO(\cdot)}$ suppresses polynomial dependence on $\ln(1/\delta)$. If
$np^\star_\midx\ge 8\log(2/\delta)$ for all $j$, then, for all $k$,
\[
\cons_\cidx(\Qh)
     \le c_k +\frac{1\!+\!2\nu}{B} +\!\sum_{\midx \in \mIdx} \abs{M_{k,j}}\, \tO\BigParens{ (np^\star_\midx)^{-\alpha}}.
\]
\end{theorem}


In other words, the solution returned by \Alg{EG:general} achieves the
lowest feasible classification error on the true distribution up to the
optimization error, which grows linearly with $\nu$, and the
statistical error, which grows as $n^{-\alpha}$. Therefore, if we want
to guarantee that the optimization error does not dominate the
statistical error, we should set $\nu\propto n^{-\alpha}$.
The fairness constraints on the true distribution are satisfied up to
the optimization error $(1+2\nu)\,\,/B$ and up to the statistical
error. Because the statistical error depends on the moments, and the
error in estimating the moments grows as $n_j^{-\alpha}\ge
n^{-\alpha}$, we can set $B\propto n^{\alpha}$ to guarantee that the
optimization error does not dominate the statistical error.
Combining this reasoning with the learning rate setting of
\Thm{alg} yields the following theorem (proved in \App{stat}).\looseness=-1
\begin{theorem}
\label{thm:main}
Let $\rho\coloneqq\max_h\norm{\consM\hvmu(h)-\hvc}_\infty$. Let
\Assume{Rn} hold for $C'\ge 2C+2+\sqrt{\ln(4/\delta)\,/\,2}$, where
$\delta>0$.  Let $Q^\star$ minimize $\err(Q)$ subject to
$\consM\momentVec(Q) \leq \consCvec$. Then \Alg{EG:general} with
$\nu\propto n^{-\alpha}$, $B\propto n^{\alpha}$ and
$\eta\propto\rho^{-2}n^{-2\alpha}$ terminates in $O(\rho^2
n^{4\alpha}\ln{\card{\cIdx}})$ iterations and returns $\Qh$,
which with probability at least $1-(|\mIdx|+1)\delta$ satisfies
\begin{align*}
    \err(\Qh)
    &
    \le \err(Q^\star) + \tO(n^{-\alpha}),\\
\cons_\cidx(\Qh)
    &
     \le c_k +\!\sum_{\midx \in \mIdx} \abs{M_{k,j}}\,
     \tO(n_\midx^{-\alpha}) &&
     \text{for all $k$}.
\end{align*}
\end{theorem}


\begin{example}[\dep]
If $n_a$ denotes the number of training examples with $A_i=a$, then
\Assume{Rn} states that we should set
$\eps_{(a,+)}=\eps_{(a,-)}=C'(n_a^{-\alpha}+n^{-\alpha})$ and \Thm{main}
then shows that for a suitable setting of $C'$, $\nu$, $B$, and~$\eta$,
\Alg{EG:general} will return a randomized classifier $\hQ$
with the lowest feasible classification error up to $\tO(n^{-\alpha})$
while also approximately satisfying the fairness constraints
\[
   \BigAbs{\E[h(X)\given A=a]-\E[h(X)]}\le\tO(n_a^{-\alpha})
\quad
   \text{for all $a$,}
\]
where $\E$ is with respect to $(X,A,Y)$ as well as $h\sim\smash{\hQ}$.
\end{example}

\begin{example}[\eo]
Similarly, if $n_{(a,y)}$ denotes the number of examples with $A_i=a$
and $Y_i=y$ and $n_{(\star,y)}$ denotes the number of examples with
$Y_i=y$, then \Assume{Rn} states that we should set
$\eps_{(a,y,+)}=\eps_{(a,y,-)}=C'(n_{(a,y)}^{-\alpha}+n_{(\star,y)}^{-\alpha})$
and \Thm{main} then shows that for a suitable setting of $C'$, $\nu$,
$B$, and $\eta$, \Alg{EG:general} will return a randomized classifier~$\hQ$
with the lowest feasible classification error up to
$\tO(n^{-\alpha})$ while also approximately satisfying the fairness
constraints\looseness=-1
\[
   \BigAbs{\E[h(X)\given A=a, Y\!=y]-\E[h(X)\given Y\!=y]}
   \le\tO(n_{(a,y)}^{-\alpha})
\]
for all $a$, $y$. Again, $\E$ includes randomness under the
true distribution over $(X,A,Y)$ as well as $h\sim\hQ$.
\end{example}

\subsection{Grid Search}
\label{sec:grid-search}


In some situations, it is preferable to select a deterministic
classifier, even if that means a lower accuracy or a modest
violation of the fairness constraints. A set of candidate classifiers
can be obtained from the saddle point
$(Q^\dagger,\vlambda^\dagger)$. Specifically, because $Q^\dagger$ is a
minimizer of $L(Q,\vlambda^\dagger)$ and $L$ is linear in $Q$, the distribution
$Q^\dagger$
puts non-zero mass only on classifiers that are the $Q$-player's
best responses to $\vlambda^\dagger$. If we knew $\vlambda^\dagger$,
we could retrieve one such best response via the reduction to
cost-sensitive learning introduced in Section~\ref{sec:reduction}.\looseness=-1

We can compute $\vlambda^\dagger$ using \Alg{EG:general}, but
when the number of constraints is very small, as is the case for demographic
parity or equalized odds with a binary protected attribute, it is also reasonable
to consider a grid of values $\vlambda$, calculate the best response for
each value, and then select the value with the desired tradeoff
between accuracy and fairness.\looseness=-1


\begin{example}[\dep]
When the protected attribute is binary, e.g., $A\in\set{a,a'}$, then
the grid search can in fact be conducted in a single dimension.
The reduction formally takes two real-valued arguments $\lambda_a$
and $\lambda_{a'}$, and then adjusts the costs for predicting $h(X_i)=1$
by the amounts
\[
  \delta_a = \frac{\lambda_a}{p_a} - \lambda_a-\lambda_{a'}
\quad
\text{and}
\quad
  \delta_{a'} = \frac{\lambda_{a'}}{p_{a'}} - \lambda_a-\lambda_{a'},
\]
respectively, on the training examples with $A_i=a$ and $A_i=a'$.
These adjustments satisfy
$
 p_a\delta_a + p_{a'}\delta_{a'}=0
$,
so instead of searching over $\lambda_a$ and $\lambda_{a'}$,
we can carry out the grid search over $\delta_a$ alone and apply the adjustment
$\delta_{a'}=-p_a\delta_a/p_{a'}$ to the protected attribute value $a'$.

With three attribute values, e.g., $A\in\set{a,a',a''}$,
we similarly have
$p_a\delta_a + p_{a'}\delta_{a'} + p_{a''}\delta_{a''}=0$,
so it suffices to conduct grid search in two dimensions rather than three.
\end{example}

\begin{example}[\eo]
If $A\in\set{a,a'}$, we obtain the adjustment
\[
  \delta_{(a,y)}
  =
  \frac{\lambda_{(a,y)}}{p_{(a,y)}}-\frac{\lambda_{(a,y)}+\lambda_{(a',y)}}{p_{(\star,y)}}
\]
for an example with protected attribute value $a$ and label $y$, and similarly
for protected attribute value $a'$. In this case,
separately for each $y$, the adjustments satisfy
\[
  p_{(a,y)}\delta_{(a,y)}+p_{(a',y)}\delta_{(a',y)}=0,
\]
so it suffices to do the grid search over $\delta_{(a,0)}$ and $\delta_{(a,1)}$
and set the parameters for $a'$ to $\delta_{(a',y)}=-p_{(a,y)}\delta_{(a,y)}/p_{(a',y)}$.
\end{example}

\begin{figure*}
\includegraphics{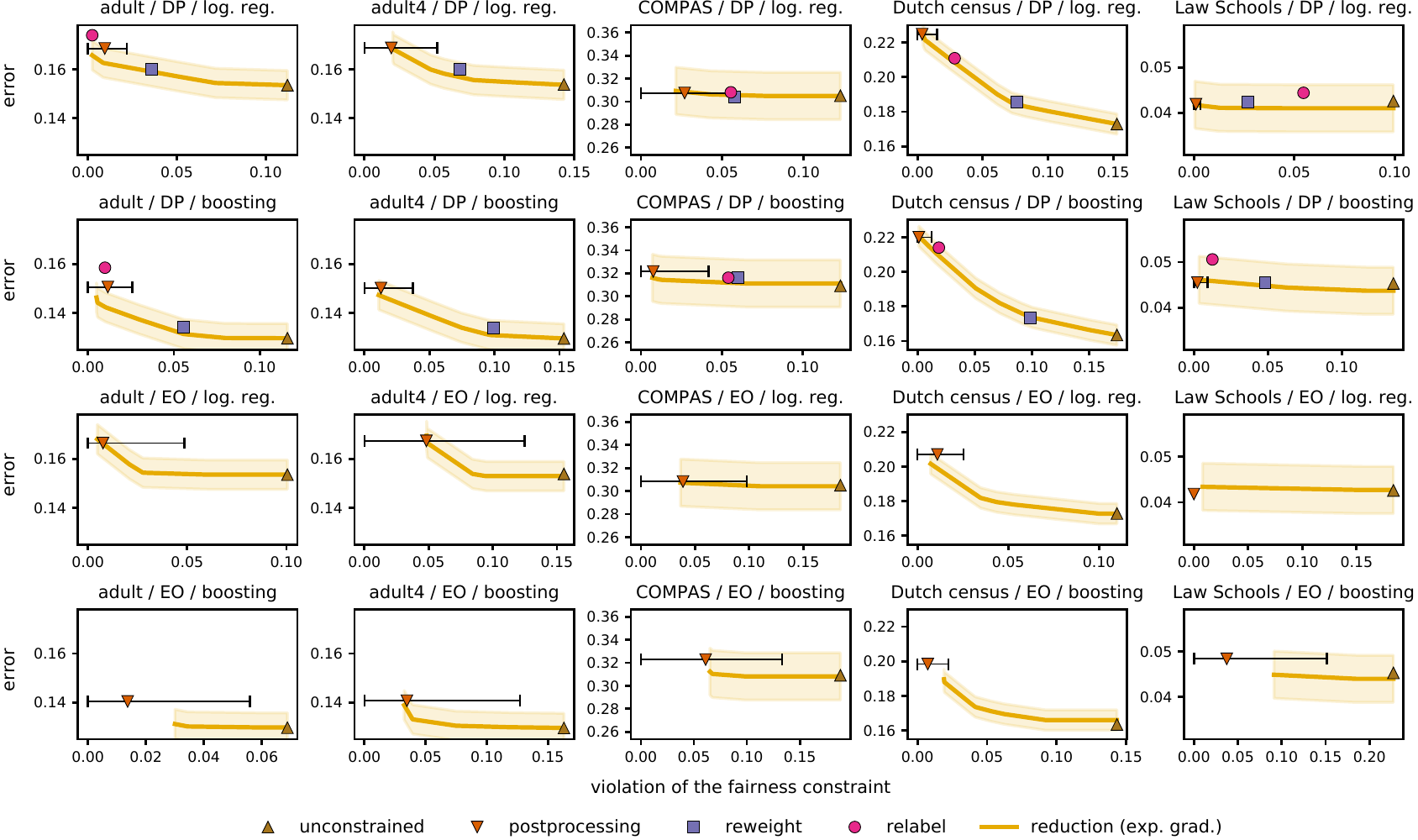}%
\vspace{-6pt}%
\caption{\emph{Test classification error versus constraint violation with respect to DP (top two rows) and EO (bottom two rows)}.
All data sets have binary protected attributes except for \textit{adult4}, which
has four protected attribute values, so relabeling is not applicable there. For our
reduction approach we plot the convex envelope of the classifiers obtained on training data at various accuracy--fairness tradeoffs.
We show 95\% confidence bands for the classification error of our reduction approach and 95\% confidence intervals for the constraint violation
of post-processing.
Our reduction approach dominates or matches the performance of the other approaches up to statistical uncertainty.\looseness=-1}
\label{fig:experiments}%
\vspace{-6pt}%
\end{figure*}

\section{Experimental Results}
\label{sec:experiments}

We now examine how our exponentiated-gradient reduction\footnote{\url{https://github.com/Microsoft/fairlearn}}
performs at the task of binary classification subject to
either demographic parity or equalized odds. We provide an evaluation
of our grid-search reduction in \App{experiments}.

We compared our reduction with the score-based post-processing
algorithm of~\citet{hardt16}, which takes as its input any classifier,
(i.e., a standard classifier without any fairness constraints) and derives a monotone transformation of the classifier's
output to remove any disparity with respect to the training examples.
This post-processing algorithm works with both demographic parity and equalized odds,
as well as with binary and non-binary protected attributes.\looseness=-1

For demographic parity, we also compared our reduction with the
\emph{reweighting} and \emph{relabeling} approaches of~\citet{kamiran12}.
Reweighting can be applied to both binary and non-binary protected attributes
and operates by changing importance weights on each example with the goal of removing any
statistical dependence between the protected attribute and label.\footnote{%
Although reweighting was developed for demographic parity, the weights that it induces are achievable by our grid search, albeit the grid search
for equalized odds rather than demographic parity.\looseness=-1}
Relabeling was developed for binary protected attributes.
First, a classifier is trained on the original data (without considering fairness). The training examples close to the decision boundary are then relabeled
to remove all disparity while minimally affecting accuracy. The final classifier
is then trained on the relabeled data.



As the base classifiers for our reductions, we used the weighted classification implementations of
logistic regression and gradient-boosted decision trees in
scikit-learn~\citep{scikit-learn}. In addition to the three baselines
described above, we also compared our reductions to the ``unconstrained'' classifiers trained to optimize accuracy only.

We used four data sets, randomly splitting each one into training examples (75\%) and test examples (25\%):
\begin{itemize}[nosep]

\item The {adult} income data set~\citep{Lichman:2013} (48,842 examples).  Here the task is to predict whether someone makes more than \$50k per year, with gender as the protected attribute. To examine the performance for non-binary protected attributes,
we also conducted another experiment with the same data, using both gender
and race (binarized into white and non-white) as the protected
attribute. Relabeling, which requires binary protected attributes, was
therefore not applicable here.

\item ProPublica's {COMPAS recidivism data} (7,918 examples).
The task is to predict recidivism from someone's criminal history, jail
and prison time, demographics, and COMPAS risk scores,
with race as the protected attribute (restricted to white and black defendants).

\item
{Law School Admissions Council's National Longitudinal Bar
  Passage Study}~\citep{Wightman} (20,649 examples). Here the task is to predict someone's eventual passage of the bar exam, with race (restricted to white and black only) as the protected attribute.

\item The Dutch census data set (Dutch Central Bureau for Statistics, 2001) (60,420 examples).  Here the task is to predict whether or not someone has a prestigious occupation, with gender as the protected attribute.
\end{itemize}

While all the evaluated algorithms require access to the protected attribute $A$ at training time,
only the post-processing algorithm requires access to $A$ at test time. For a fair comparison,
we included $A$ in the feature vector $X$, so all algorithms had access to it at both the training
time and test time.\looseness=-1

We used the test examples to measure the classification error for each
approach, as well as the
violation of the desired fairness constraints, i.e., ${\max_a\bigAbs{\E[h(X)\given {A=a}]-\E[h(X)]}}$
and $\max_{a,y}\bigAbs{\E[h(X)\given {A=a},
    {Y=y}]-\E[h(X)\given{Y=y}]}$ for demographic parity and
equalized odds, respectively.

We ran our reduction across a wide range of tradeoffs between the classification
error and fairness constraints. We considered $\eps\in\set{0.001,\dotsc,0.1}$ and for each value
ran Algorithm~\ref{alg:EG:general} with $\hc_k=\eps$ across all $k$. As expected,
the returned randomized classifiers tracked the training Pareto frontier (see \Fig{exp:all:train} in \App{experiments}).
In Figure~\ref{fig:experiments}, we evaluate these classifiers alongside the baselines on the \emph{test} data.

For all the data sets, the range of classification errors is much smaller than
the range of constraint violations. Almost all the approaches were able to substantially reduce or remove disparity without much
impact on classifier accuracy. One exception was the Dutch census data
set, where the classification error increased the most in relative terms.\looseness=-1

Our reduction
generally dominated or matched the baselines. The relabeling approach frequently yielded solutions that were not
Pareto optimal. Reweighting yielded solutions on the Pareto frontier, but
often with substantial disparity. As expected, post-processing yielded
disparities that were statistically
indistinguishable from zero, but the resulting classification error was sometimes higher
than achieved by our reduction
under a statistically indistinguishable disparity.
In addition, and unlike the
post-processing algorithm, our reduction can achieve any
desired accuracy--fairness tradeoff, allows a wider range of fairness definitions, and does not require access to the protected attribute at test time.


Our grid-search reduction, evaluated in \App{experiments}, sometimes failed to achieve the
lowest disparities on the training data, but its performance on the test data very closely
matched that of our exponentiated-gradient reduction.
However, if the protected attribute is non-binary, then grid search is not feasible.
For instance, for the version of the adult income data set where the protected attribute takes on four values,
the grid search would need to span three dimensions for demographic
parity and six dimensions for
equalized odds, both of which are prohibitively costly.\looseness=-1

\section{Conclusion}

We presented two reductions for achieving fairness in a
binary classification setting. Our reductions work for any
classifier representation, encompass many
definitions of fairness, satisfy provable guarantees, and work well
in practice.\looseness=-1

Our reductions optimize the tradeoff between accuracy and any (single)
definition of fairness given training-time access to protected attributes.
Achieving fairness when training-time access to protected
attributes is unavailable remains an open problem for future research,
as does the navigation of tradeoffs between accuracy and multiple fairness
definitions.

\section*{Acknowledgements}

We would like to thank Aaron Roth, Sam Corbett-Davies, and Emma Pierson for helpful discussions.



\bibliography{refs}
\bibliographystyle{icml2018}

\newpage
\onecolumn
\appendix
\section{Error and Fairness for Randomized Classifiers}
\label{app:randomized}


Let $D$ denote the distribution over triples $(X,A,Y)$.
The accuracy of a classifier $h\in\calH$ is measured by 0-1 error,
$\err(h)\coloneqq\P_D[h(X)\ne Y]$,
which for a randomized classifier $Q$ becomes
\[
  \err(Q)\coloneqq
\mathop{\P}_{(X,A,Y)\sim D,\,h\sim Q}[h(X)\ne Y]
  =\!\sum_{h\in\calH} Q(h)\,\err(h)
\enspace
.
\]
The fairness constraints on a classifier $h$ are $\consM\vmu(h)\le\consCvec$. Recall
that $\mu_j(h)\coloneqq\E_D[g_j(X,A,Y,h(X))\given\event_j]$. For
a randomized classifier $Q$ we define its moment $\mu_j$ as
\[
  \mu_j(Q)
  \coloneqq \mathop{\E}_{(X,A,Y)\sim D,\,h\sim Q}
  \BigBracks{g_j(X,A,Y,h(X))\BigGiven\event_j}
  =\!\sum_{h\in\calH} Q(h)\mu_j(h)
\enspace,
\]
where the last equality follows because $\event_j$ is independent of the
choice of $h$.
%

\section{Proof of \Thm{alg}}
\label{app:alg}

The proof follows immediately from the analysis of \citet{FreundSc96} applied to the Exponentiated
Gradient (EG) algorithm~\citep{KivinenWa97}, which in our specific case is also equivalent to Hedge~\citep{FreundSc97}.

Let $\Lambda\coloneqq\set{\vlambda\in\R^{\card{\cIdx}}_+:\:\norm{\vlambda'}_1\le B}$
and $\Lambda'\coloneqq\set{\vlambda'\in\R^{\card{\cIdx}+1}_+:\:\norm{\vlambda'}_1=B}$.
We associate any $\vlambda\in\Lambda$ with the $\vlambda'\in\Lambda'$ that is equal to
$\vlambda$ on coordinates $1$ through $\card{\cIdx}$ and puts the remaining mass on the coordinate $\lambda'_{\card{\cIdx}+1}$.

Consider a run of \Alg{EG:general}. For each $\vlambda_t$, let $\vlambda'_t\in\Lambda'$ be the associated element of $\Lambda'$.
Let $\vr_t\coloneqq\consM\hvmu(h_t)-\hvc$ and let $\vr'_t\in\R^{\card{\cIdx}+1}$ be equal to $\vr_t$ on coordinates $1$ through $\card{\cIdx}$
and put zero on the coordinate $r'_{t,\card{\cIdx}+1}$. Thus, for any $\vlambda$ and the associated $\vlambda'$, we have, for all $t$,
\begin{equation}
\label{eqn:vlambda'}
   \vlambda\trans\vr_t=(\vlambda')\trans\vr'_t
\enspace,
\end{equation}
and, in particular,
\begin{equation}
\label{eqn:vlambda'_t}
  \vlambda_t\trans\bigParens{\consM\hvmu(h_t)-\hvc}
  =
  \vlambda_t\trans\vr_t
  =
  (\vlambda'_t)\trans\vr'_t
\enspace.
\end{equation}
We interpret $\vr'_t$ as the reward vector for the $\vlambda$-player. The choices of $\vlambda'_t$ then correspond
to those of the EG algorithm with the learning rate $\eta$.
By the assumption of the theorem we have $\norm{\vr'_t}_\infty=\norm{\vr_t}_\infty\le\rho$.
The regret bound for EG, specifically,
Corollary 2.14 of \citet{ShalevShwartz12}, then states that for any $\vlambda'\in\Lambda'$,
\[
   \sum_{t=1}^T (\vlambda')\trans\vr'_t
\le
   \sum_{t=1}^T (\vlambda'_t)\trans\vr'_t + \underbrace{\frac{B\log(\card{\cIdx}+1)}{\eta}+\eta\rho^2 B T}_{\eqqcolon\zeta_T}
\enspace.
\]
Therefore, by \eqns{vlambda'}{vlambda'_t}, we also have for any $\vlambda\in\Lambda$,
\begin{equation}
\label{eqn:EG:regret}
   \sum_{t=1}^T \vlambda\trans\vr_t
\le
   \sum_{t=1}^T \vlambda_t\trans\vr_t + \zeta_T
\enspace.
\end{equation}
This regret bound can be used to bound the suboptimality of $L(\hQ_T,\hvlambda_T)$ in $\hvlambda_T$
as follows:
\begin{align}
\notag
  L(\hQ_T,\vlambda)
&=\frac{1}{T}
  \sum_{t=1}^T\BigParens{\herr(h_t)+\vlambda\trans\bigParens{\consM\hvmu(h_t)-\hvc}}
\\
\notag
&=\frac{1}{T}
  \sum_{t=1}^T\BigParens{\herr(h_t)+\vlambda\trans\vr_t}
\\
\label{eqn:nu:1}
&\le\frac{1}{T}
  \sum_{t=1}^T\BigParens{\herr(h_t)+\vlambda_t\trans\vr_t}
  +\frac{\zeta_T}{T}
\\
\notag
&=\frac{1}{T}
  \sum_{t=1}^T L(h_t,\vlambda_t)
  +\frac{\zeta_T}{T}
\\
\label{eqn:nu:2}
&\le\frac{1}{T}
  \sum_{t=1}^T L(\hQ_T,\vlambda_t)
  +\frac{\zeta_T}{T}
\\
\label{eqn:nu:3}
&=L\BigParens{\hQ_T,\,\frac{1}{T}\sum_{t=1}^T\vlambda_t}+\frac{\zeta_T}{T}
 =L(\hQ_T,\hvlambda_T)+\frac{\zeta_T}{T}
\enspace.
\end{align}
\Eqn{nu:1} follows from the regret bound~\eqref{eqn:EG:regret}. \Eqn{nu:2} follows because $L(h_t,\vlambda_t)\le L(Q,\vlambda_t)$ for all $Q$ by the choice of $h_t$ as the best response
of the $Q$-player. Finally, \eqn{nu:3} follows by linearity of $L(Q,\vlambda)$ in $\vlambda$. Thus, we have for all $\vlambda\in\Lambda$,
\begin{equation}
\label{eqn:nu:second}
  L(\hQ_T,\hvlambda_T)
  \ge
  L(\hQ_T,\vlambda)-\frac{\zeta_T}{T}
\enspace.
\end{equation}
Also, for any $Q$,
\begin{align}
\label{eqn:nu:4}
  L(Q,\hvlambda_T)
&=\frac{1}{T}
  \sum_{t=1}^T L(Q,\vlambda_t)
\\
\label{eqn:nu:5}
&\ge\frac{1}{T}
  \sum_{t=1}^T L(h_t,\vlambda_t)
\\
\label{eqn:nu:6}
&\ge\frac{1}{T}
  \sum_{t=1}^T L(h_t,\hvlambda_T)-\frac{\zeta_T}{T}
\\
\label{eqn:nu:7}
&=
  L(\hQ_T,\hvlambda_T)-\frac{\zeta_T}{T}
\enspace,
\end{align}
where \eqn{nu:4} follows by linearity of $L(Q,\vlambda)$ in $\vlambda$, \eqn{nu:5} follows by the optimality of $h_t$
with respect to $\hvlambda_t$, \eqn{nu:6} from the regret bound~\eqref{eqn:EG:regret}, and \eqn{nu:7} by linearity of $L(Q,\vlambda)$
in $Q$. Thus, for all $Q$,
\begin{equation}
\label{eqn:nu:first}
  L(\hQ_T,\hvlambda_T)
  \le
  L(Q,\hvlambda_T)+\frac{\zeta_T}{T}
\enspace.
\end{equation}
\Eqns{nu:second}{nu:first} immediately imply that for any $T\ge 1$,
\[
  \nu_T\le\frac{\zeta_T}{T}=\frac{B\log(\card{\cIdx}+1)}{\eta T}+\eta\rho^2 B
\enspace,
\]
proving the first part of the theorem.

The second part of the theorem follows by plugging in $\eta=\frac{\nu}{2\rho^2 B}$
and verifying that if $T\ge\frac{4\rho^2 B^2\log(\card{\cIdx}+1)}{\nu^2}$ then
\[
  \nu_T
\le\frac{B\log(\card{\cIdx}+1)}{\frac{\nu}{2\rho^2 B}\cdot\frac{4\rho^2 B^2\log(\card{\cIdx}+1)}{\nu^2}}+\frac{\nu}{2\rho^2 B}\cdot\rho^2 B
=\frac{\nu}{2}+\frac{\nu}{2}
\enspace.
\]

\section{Proofs of Theorems~\ref{thm:stat} and~\ref{thm:main}}
\label{app:stat}

The bulk of this appendix proves the following theorem, which will immediately imply Theorems~\ref{thm:stat} and~\ref{thm:main}.
\begin{theorem}
\label{thm:app:stat}
Let $(\Qh, \vlambdah)$ be any $\nu$-approximate saddle point of $L$ with
\[
    \consCh_\cidx = \consC_\cidx + \eps_k
\quad
\text{and}
\quad
    \eps_k\ge\sum_{\midx \in \mIdx} \abs{M_{k,j}}\Parens{
             2R_{n_\midx}(\calH) + \frac{2}{\sqrt{n_\midx}} + \sqrt{\frac{\ln(2/\delta)}{2n_\midx}}
    }
\enspace.
\]
Let $Q^\star$ minimize $\err(Q)$ subject to $\consM\momentVec(Q) \leq \consCvec$. Then with probability at least $1-(|\mIdx|+1)\delta$, the distribution $\Qh$ satisfies
\begin{align*}
    \err(\hQ)
    &
    \le
    \err(Q^\star) + 2\nu
    +
    4R_n(\calH) + \frac{4}{\sqrt{n}} + \sqrt{\frac{2\ln(2/\delta)}{n}}
\enspace,
\\
\text{and for all $\cidx$,}
\quad
\cons_\cidx(\Qh)
    &
    \le
    c_k +\frac{1+2\nu}{B} + 2\eps_k
\enspace.
\end{align*}
\end{theorem}

Let $\Lambda\coloneqq\set{\vlambda\in\R^{\card{\cIdx}}_+:\:\norm{\vlambda'}_1\le B}$ denote the domain of $\vlambda$.
In the remainder of the section, we assume that we are given a pair $(\Qh, \vlambdah)$
which is a $\nu$-approximate saddle point of $L$, i.e.,
\begin{equation}
\label{eqn:nu-opt}
\begin{aligned}
  L(\Qh, \vlambdah)
&\leq L(Q, \vlambdah) + \nu
\quad
  \text{for all $Q\in\Delta$,}
\\
\text{and}
\quad
  L(\Qh,\vlambdah)
&\geq L(\Qh, \vlambda) - \nu
\quad
  \text{for all $\vlambda\in\Lambda$.}
\end{aligned}
\end{equation}
We first establish that the pair $(\Qh, \vlambdah)$ satisfies an approximate version of complementary slackness. For the statement and proof
of the following lemma, recall that $\hvgamma(Q)=\consM\hvmu(Q)$, so the empirical fairness constraints
can be written as $\hvgamma(Q)\le\hvc$ and the Lagrangian $L$ can be written as
\begin{equation}
\label{eqn:L:gamma}
   L(Q,\vlambda)
   =
   \herr(Q)+\sum_{\cidx\in\cIdx}\lambda_\cidx(\hgamma_\cidx(Q)-\hc_\cidx)
\enspace.
\end{equation}
\begin{lemma}[Approximate complementary slackness]
\label{lemma:slackness}
The pair $(\Qh, \vlambdah)$ satisfies
\[
    \sum_{\cidx \in \cIdx} \lambdah_\cidx(\consh_\cidx(\Qh) - \consCh_\cidx) \geq B\max_{\cidx \in \cIdx} \bigParens{\consh_\cidx(\Qh) - \consCh_\cidx}_+ - \nu
\enspace,
\]
where we abbreviate $x_+ = \max\set{x,0}$ for any real number $x$.
\end{lemma}

\begin{proof}
    We show that the lemma follows from the optimality conditions~\eqref{eqn:nu-opt}. We consider a dual variable $\vlambda$ defined as
    \begin{align*}
        \vlambda =
        \begin{cases}
        \vzero
           &\text{if $\hvgamma(\hQ)\le\hvc$,}
        \\
        B\ve_{\cidx^\star}
           &\text{otherwise, where $\cidx^\star=\argmax_\cidx[\hgamma_\cidx(\hQ)-\hc_\cidx]$,}
        \end{cases}
    \end{align*}
    where $\ve_k$ denotes the $k$th vector of the standard basis.
    Then we have by \eqns{nu-opt}{L:gamma} that
    \begin{align*}
        \herr(\Qh) + \sum_{\cidx \in \cIdx} \lambdah_\cidx(\consh_\cidx(\Qh) - \consCh_\cidx)
&
        =   L(\Qh,\hvlambda)
\\[-6pt]
&
        \ge L(\Qh,\vlambda) - \nu
        =   \herr(\Qh) + \sum_{\cidx \in \cIdx} \lambda_\cidx(\consh_\cidx(\Qh) - \consCh_\cidx) - \nu
\enspace,
    \end{align*}
    and the lemma follows by our choice of $\vlambda$.
\end{proof}

Next two lemmas bound the empirical error of $\hQ$ and
also bound the amount by which $\hQ$ violates the empirical fairness constraints.

\begin{lemma}[Empirical error bound]
\label{lemma:regret}
The distribution $\Qh$ satisfies $\herr(\Qh)\le\herr(Q)+2\nu$ for any $Q$
satisfying the empirical fairness constraints, i.e., any $Q$ such that $\hvgamma(Q)\le\hvc$.
\end{lemma}
\begin{proof}
Assume that $Q$ satisfies $\hvgamma(Q)\le\hvc$. Since $\vlambdah\ge\vzero$, we have
\[
L(Q, \hvlambda)=\herr(Q)+ \hvlambda\trans\bigParens{\hvgamma(Q)-\hvc} \le\herr(Q)
\enspace.
\]
The optimality conditions~\eqref{eqn:nu-opt} imply that
\[
  L(\Qh, \vlambdah) \leq L(Q, \hvlambda) + \nu
\enspace.
\]
Putting these together, we obtain
\begin{align*}
L(\hQ,\hvlambda)
\le
\herr(Q)+\nu
\enspace.
\end{align*}
We next invoke Lemma~\ref{lemma:slackness} to lower bound $L(\hQ,\hvlambda)$ as
\begin{align*}
L(\hQ,\hvlambda)
=
\herr(\Qh) + \sum_{\cidx \in \cIdx} \lambdah_\cidx(\consh_\cidx(\Qh) - \consCh_\cidx)
&\geq \herr(\Qh) + B\max_{\cidx \in \cIdx}\bigParens{\consh_\cidx(\Qh) - \consCh_\cidx}_+ - \nu
\\
&\geq \herr(\Qh) - \nu
\enspace.
\end{align*}
Combining the upper and lower bounds on $L(\hQ,\hvlambda)$ completes the proof.
\end{proof}

\begin{lemma}[Empirical fairness violation]
\label{lemma:constraints}
Assume that the empirical fairness constraints $\hvgamma(Q)\le\hvc$ are feasible.
Then the distribution $\Qh$ approximately satisfies all empirical fairness constraints:
\[
    \max_{\cidx \in \cIdx} \left(\consh_\cidx(\Qh) - \consCh_\cidx\right) \leq \frac{1+2\nu}{B}
\enspace.
\]
\end{lemma}

\begin{proof}
    Let $Q$ satisfy $\hvgamma(Q)\le\hvc$. Applying the same upper and lower bound
    on $L(\Qh,\vlambdah)$ as in the proof of Lemma~\ref{lemma:regret}, we obtain
    \[
        \herr(\Qh) + B\max_{\cidx \in \cIdx}\bigParens{\consh_\cidx(\Qh) - \consCh_\cidx}_+ - \nu
        \;\leq\;
        L(\Qh,\vlambdah)
        \;\leq\;
        \herr(Q) + \nu
\enspace.
    \]
    We can further upper bound $\herr(Q)-\herr(\Qh)$ by 1 and use $x \leq x_+$ for any real number $x$ to complete the proof.
\end{proof}

It remains to lift the bounds on empirical classification error and constraint violation into the corresponding bounds on true classification error and the violation
of true constraints. We will use the standard machinery of uniform convergence bounds via the (worst-case) Rademacher complexity.

Let $\calF$ be a class of functions $f:\calZ\to[0,1]$ over some space $\calZ$. Then the (worst-case) \emph{Rademacher complexity} of $\calF$ is defined as
\[
   R_n(\calF)
\coloneqq
   \sup_{z_1,\dotsc,z_n\in\calZ}
   \E\Bracks{
   \sup_{f\in\calF}\Abs{
      \frac{1}{n} \sum_{i=1}^n
        \sigma_i f(z_i)
   }}
\enspace,
\]
where the expectation is over the i.i.d.\ random variables $\sigma_1,\dotsc,\sigma_n$ with $\P[\sigma_i=1]=\P[\sigma_i=-1]=1/2$.

We first prove concentration of generic moments derived from classifiers $h\in\calH$ and then move to bounding the deviations from true classification error
and true fairness constraints.
\begin{lemma}[Concentration of moments]
\label{lemma:dev}
Let $g:\calX\times\Attr\times\set{0,1}\times\set{0,1}\to[0,1]$ be any function and let $D$ be a distribution over $(X,A,Y)$.
Then with probability at least $1-\delta$, for all $h\in\calH$,
  \[
  \BigAbs{
     \Ehat\bigBracks{g(X,A,Y,h(X))}-\E\bigBracks{g(X,A,Y,h(X))}
  }
  \le
     2R_n(\calH) + \frac{2}{\sqrt{n}} + \sqrt{\frac{\ln(2/\delta)}{2n}}
\enspace,
  \]
where the expectation is with respect to $D$ and the empirical expectation is based on $n$ i.i.d.\ draws from $D$.
\end{lemma}
\begin{proof}
Let $\calF\coloneqq\set{f_h}_{h\in\calH}$ be the class of functions $f_h: (x,y,a)\mapsto g\bigParens{x,y,a,h(x)}$.
By Theorem 3.2 of \citet{BoucheronBoLu05}, we then
have with probability at least $1-\delta$, for all $h$,
\begin{equation}
\label{eqn:fh:bound}
  \BigAbs{
     \Ehat\bigBracks{g(X,A,Y,h(X))}-\E\bigBracks{g(X,A,Y,h(X))}
  }
  =
  \BigAbs{\Ehat[f_h]-\E[f_h]}
  \le
  2 R_n(\calF)
  +
  \sqrt{\frac{\ln(2/\delta)}{2n}}
\enspace.
\end{equation}
We will next bound $R_n(\calF)$ in terms of $R_n(\calH)$.
Since $h(x)\in\set{0,1}$, we can write
\[
  f_h(x,y,a)=h(x) g(x,a,y,1) + \BigParens{1-h(x)}g(x,a,y,0)
            =g(x,a,y,0)+h(x)\BigParens{g(x,a,y,1)-g(x,a,y,0)}
  \enspace.
\]
Since $\bigAbs{g(x,a,y,0)}\le 1$ and $\bigAbs{g(x,a,y,1)-g(x,a,y,0)}\le 1$, we can invoke Theorem 12(5) of \citet{BartlettMe02} for bounding
function classes shifted by an offset, in our case $g(x,a,y,0)$, and Theorem 4.4 of \citet{LedouxTa91} for bounding function classes
under contraction, in our case $g(x,a,y,1)-g(x,a,y,0)$, yielding
\[
  R_n(\calF)\le \frac{1}{\sqrt{n}} + R_n(\calH)
\enspace.
\]
Together with the bound~\eqref{eqn:fh:bound}, this proves the lemma.
\end{proof}

\begin{lemma}[Concentration of loss]
\label{lemma:dev:loss}
  With probability at least $1-\delta$, for all $Q\in\Delta$,
  \[
  \left| \herr(Q)-\err(Q) \right|
  \le
     2R_n(\calH) + \frac{2}{\sqrt{n}} + \sqrt{\frac{\ln(2/\delta)}{2n}}
\enspace.
  \]
\end{lemma}

\begin{proof}
We first use
\Lemma{dev} with $g:(x,a,y,\hat{y})\mapsto\ind\set{y\ne\hat{y}}$ to obtain, with probability $1-\delta$, for all $h$,
\[
  \BigAbs{\herr(h)-\err(h)}
  =
  \BigAbs{\Ehat[f_h]-\E[f_h]}
  \le
     2R_n(\calH) + \frac{2}{\sqrt{n}} + \sqrt{\frac{\ln(2/\delta)}{2n}}
\enspace.
\]
The lemma now follows for any $Q$ by taking a convex combination of the corresponding bounds on $h\in\calH$.\footnote{%
The same reasoning applies for general error, $\err(h) = \E[g_\err(X\!,\!A,\!Y\!,\!h(X))]$, by using $g=g_\err$ in
\Lemma{dev}.}
\end{proof}

Finally, we show a result for the concentration of the empirical constraint violations to their population counterparts. We will actually show the concentration of the individual moments $\momenth_\midx(Q)$ to $\moment_\midx(Q)$ uniformly for all $Q\in\Delta$. Since $\consM$ is a fixed matrix not dependent on the data, this also directly implies concentration of the constraints $\hvgamma(Q)=\consM\hvmu(Q)$ to $\vgamma(Q)=\consM\vmu(Q)$. For this result, recall that
$n_j=\card{\set{i\in[n]:\:(X_i,\!A_i,\!Y_i)\in\event_\midx}}$ and $p^\star_\midx = \P[\event_\midx]$.

\begin{lemma}[Concentration of conditional moments]
\label{lemma:dev:cons}
    For any $\midx\in\mIdx$, with probability at least $1-\delta$, for all $Q$,
    \[
        \bigAbs{\momenth_\midx(Q) - \moment_\midx(Q)}
        \leq
        2R_{n_\midx}(\calH) + \frac{2}{\sqrt{n_\midx}} + \sqrt{\frac{\ln(2/\delta)}{2n_\midx}}
\enspace.
    \]
    If $np^\star_\midx\ge 8\log(2/\delta)$, then with probability at least $1-\delta$, for all $Q$,
    \[
        \bigAbs{\momenth_\midx(Q) - \moment_\midx(Q)}
        \leq
        2R_{np^\star_\midx/2}(\calH) + 2\sqrt{\frac{2}{np^\star_\midx}} + \sqrt{\frac{\ln(4/\delta)}{np^\star_\midx}}
\enspace.
    \]
\end{lemma}

\begin{proof}
Our proof largely follows the proof of Lemma 2 of \citet{woodworth17}, with appropriate modifications for our more general constraint definition. Let $S_\midx\coloneqq\set{i\in[n]:\:(X_i,\!A_i,\!Y_i)\in\event_\midx}$ be the set of indices such that the corresponding examples fall in the event $\event_\midx$. Note that we have defined $n_\midx = |S_\midx|$. Let $D(\cdot)$ denote the joint distribution of $(X,A,Y)$. Then, conditioned on $i \in S_\midx$, the random variables $g_\midx(X_i,\!A_i,\!Y_i,\!h(X_i))$ are i.i.d.\ draws from
the distribution $D(\cdot\given\event_j)$, with mean $\moment_\midx(h)$. Applying \Lemma{dev} with $g_\midx$ and the distribution $D(\cdot\given\event_j)$ therefore yields, with probability $1-\delta$, for all $h$,
\[
        \bigAbs{\momenth_\midx(h) - \moment_\midx(h)}
        \leq
        2R_{n_\midx}(\calH) + \frac{2}{\sqrt{n_\midx}} + \sqrt{\frac{\ln(2/\delta)}{2n_\midx}}
\enspace,
\]
The lemma now follows by taking a convex combination over $h$.
\end{proof}

\begin{proof}[Proof of \Thm{app:stat}]
    We now use the lemmas derives so far to prove \Thm{app:stat}. We first use \Lemma{dev:cons} to bound the gap
    between the empirical and population fairness constraints. The lemma implies that with probability at least $1-\card{\mIdx}\delta$, for all $\cidx \in \cIdx$ and all $Q\in\Delta$,
    \begin{align}
\notag
    \bigAbs{\consh_\cidx(Q) - \cons_\cidx(Q)}
    &= \BigAbs{\consM_\cidx \BigParens{\momentVech(Q) - \momentVec(Q)} }
    \\
\notag
        &\leq \sum_{\midx \in \mIdx} \abs{M_{k,j}}\BigAbs{\hmu_j(Q)-\mu_j(Q)}
    \\
\notag
        &\leq \sum_{\midx \in \mIdx} \abs{M_{k,j}} \left(
           2R_{n_\midx}(\calH) + \frac{2}{\sqrt{n_\midx}} + \sqrt{\frac{\ln(2/\delta)}{2n_\midx}}
        \right)
    \\
\label{eqn:epsk-dev}
        &\le\eps_k
\enspace.
    \end{align}
    Note that our choice of $\consCvech$ along with \eqn{epsk-dev} ensure that $\consh_\cidx(Q^\star) \leq \consCh_\cidx$ for all $\cidx \in \cIdx$. Using Lemma~\ref{lemma:regret} allows us to conclude that
    \[
        \herr(\hQ)\le\herr(Q^\star)+2\nu
\enspace.
    \]
    We now invoke \Lemma{dev:loss} twice, once for $\herr(\Qh)$ and once for $\herr(Q^\star)$, proving the first statement of the theorem.

    The above shows that $Q^\star$ satisfies the empirical fairness constraints, so we can use Lemma~\ref{lemma:constraints}, which together
    with \eqn{epsk-dev} yields
     \[
       \cons_\cidx(\hQ)
       \le \consh_\cidx(\hQ) + \eps_k
       \le \hc_\cidx+\frac{1+2\nu}{B}+\eps_k=c_\cidx+\frac{1+2\nu}{B}+2\eps_k
     \enspace,
     \]
     proving the second statement of the theorem.
\end{proof}

We are now ready to prove Theorems~\ref{thm:stat} and~\ref{thm:main}

\begin{proof}[Proof of \Thm{stat}]
%
%
The first part of the theorem follows immediately from \Assume{Rn} and \Thm{app:stat} (with $\delta/2$ instead of $\delta$). The statement in fact holds with probability at least $1-(\card{\mIdx}+1)\delta/2$. For the second part, we use the multiplicative Chernoff bound for binomial random variables. Note that $\E[n_\midx]=np^\star_\midx$, and we assume that $np^\star_\midx\ge 8\ln(2/\delta)$, so the multiplicative Chernoff bound implies that $n_\midx\le np^\star_\midx/2$ with probability at most $\delta/2$. Taking the union bound across all $\midx$ and combining with the first part of the theorem then proves the second part.
\end{proof}

\begin{proof}[Proof of \Thm{main}]
This follows immediately from \Thm{alg} and the first part of \Thm{stat}.
\end{proof}


\begin{figure*}
\includegraphics{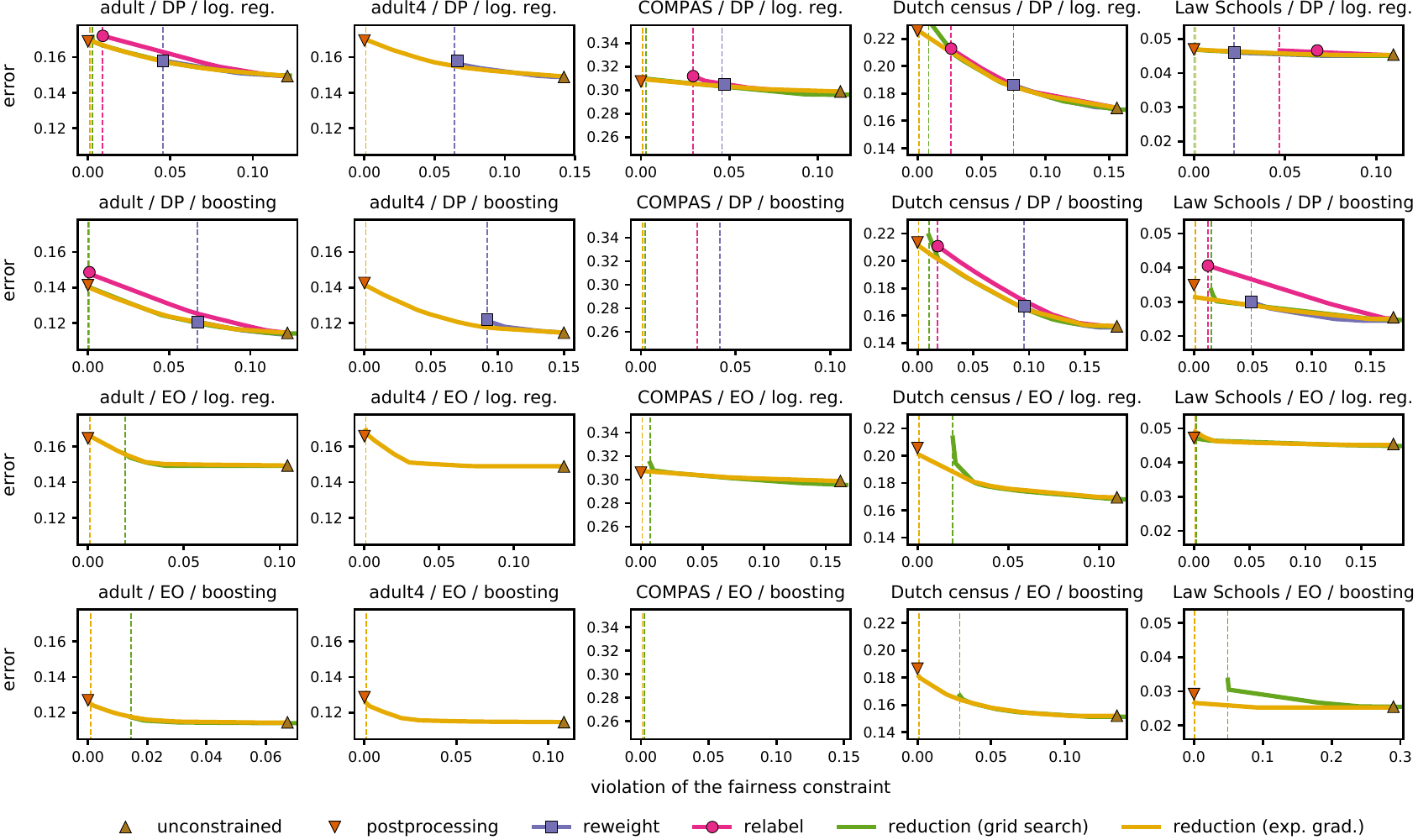}%
\vspace{-6pt}%
\caption{\emph{Training classification error versus constraint violation, with respect to DP (top two rows) and EO (bottom two rows)}. Markers
correspond to the baselines. For our two reductions and the interpolants between reweighting (or relabeling) and the unconstrained classifier,
we varied their tradeoff parameters and plot the Pareto frontiers of the sets of classifiers obtained for each method.
Because the curves of the different methods
often overlap, we use vertical dashed lines to indicate the lowest constraint violations. All data sets have binary protected attributes except for \textit{adult4}, which
has four protected attribute values, so relabeling is not applicable and grid search is not feasible for this data set.
The exponentiated-gradient reduction dominates or matches other approaches as expected since it solves exactly for the points on the Pareto frontier of the set of all classifiers
in each considered class.\\
\looseness=-1}
\label{fig:exp:all:train}%
\vspace{-6pt}%
\end{figure*}

\begin{figure*}
\includegraphics{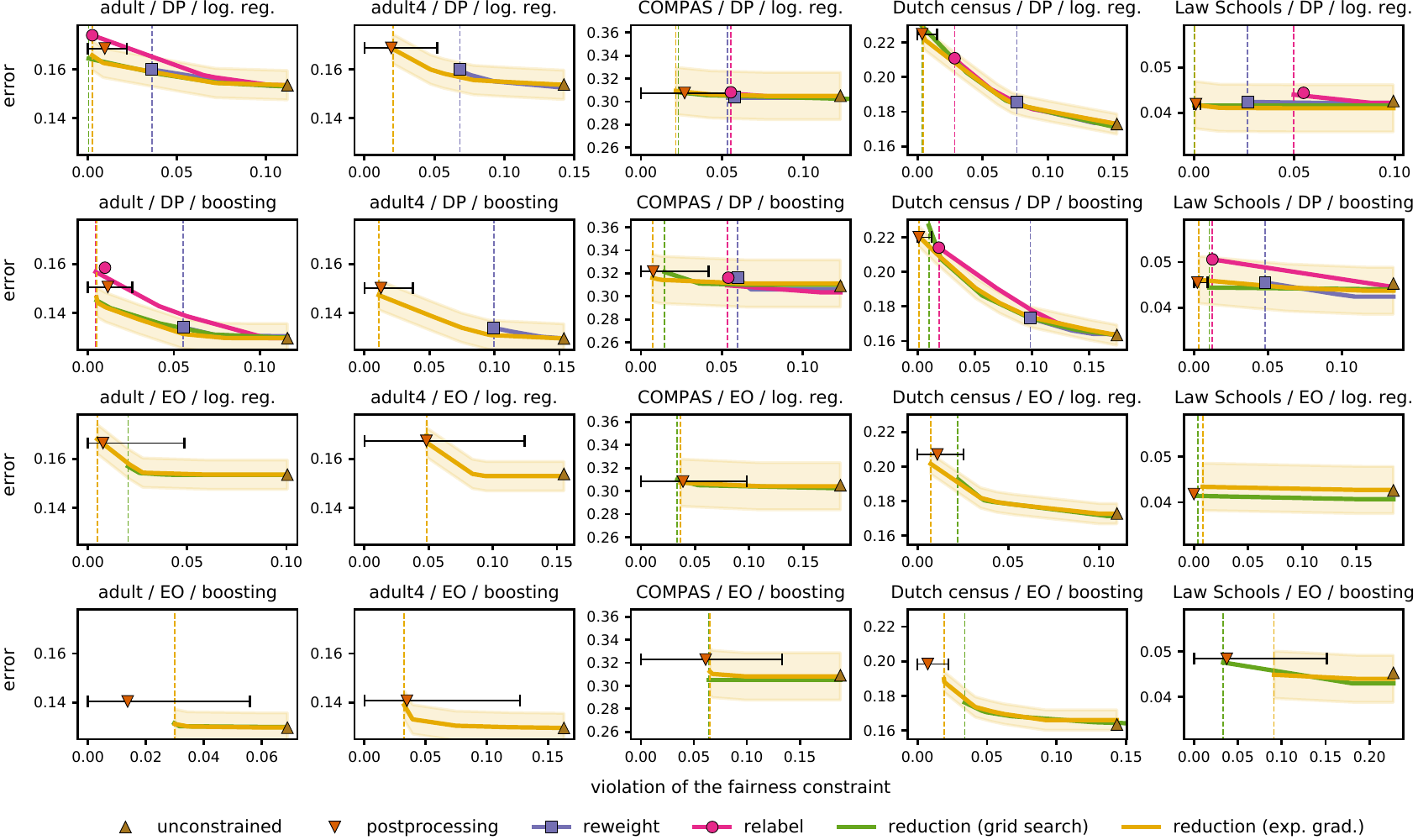}%
\vspace{-6pt}%
\caption{\emph{Test classification error versus constraint violation, with respect to DP (top two rows) and EO (bottom two rows)}. Markers
correspond to the baselines. For our two reductions and the interpolants between reweighting (or relabeling) and the unconstrained classifier,
we show convex envelopes of the classifiers taken from the \emph{training} Pareto frontier of each method (i.e., the same classifiers as shown in \Fig{exp:all:train}).
Because the curves of the different methods
often overlap, we use vertical dashed lines to indicate the lowest constraint violations. All data sets have binary protected attributes except for \textit{adult4}, which
has four protected attribute values, so relabeling is not applicable and grid search is not feasible for this data set.
We show 95\% confidence bands for the classification error of the exponentiated-gradient reduction and 95\% confidence intervals for the constraint violation
of post-processing.
The exponentiated-gradient reduction dominates or matches performance of all other methods up to statistical uncertainty.\looseness=-1}
\label{fig:exp:all:test}%
\vspace{-6pt}%
\end{figure*}

\section{Additional Experimental Results}
\label{app:experiments}

In this appendix we present more complete experimental results. We
present experimental results for both the training and test data.
We evaluate the exponentiated-gradient as well as the grid-search variants
of our reductions. And, finally, we consider extensions of reweighting
and relabeling beyond the specific tradeoffs proposed by~\citet{kamiran12}.
Specifically, we introduce a scaling parameter that interpolates between the prescribed tradeoff
(specific importance weights or the number of examples to relabel) and the unconstrained classifier
(uniform weights or zero examples to relabel). The training data results
are shown in \Fig{exp:all:train}. The test set results are shown in \Fig{exp:all:test}.

\end{document}